\documentclass[manuscript, sigconf]{acmart}

\pdfoutput=1

\usepackage{wrapfig}
\usepackage{color}
\usepackage{booktabs}
\usepackage{listings}
\usepackage{array}
\usepackage{tabularx}
\usepackage{xspace}
\usepackage{balance}
\usepackage{multirow}
\usepackage{amsmath}
\usepackage{hyperref}
\usepackage[utf8]{inputenc}
\usepackage[ruled,vlined, linesnumbered]{algorithm2e}

\SetKwComment{Comment}{$\triangleright$\ }{}
\newtheorem{theorem}{Theorem}

\newcommand{\etal}{et al. }
\newcommand{\numOfInteractionsFinal}{17,386,309}
\newcommand{\numOfUsersFinal}{75,040}
\newcommand{\numOfItemsFinal}{9,781}

\newcommand{\mfnpTwo}{0.4310} 
\newcommand{\mfnpFive}{0.6700} 
\newcommand{\mfnpTen}{0.8179}

\newcommand{\dplcfTwo}{0.4203}
\newcommand{\dplcfFive}{0.5969} 
\newcommand{\dplcfTen}{0.7000} 

\newcommand{\fmfldpLargeTwof}{0.2163}
\newcommand{\fmfldpLargeFivef}{0.3755}
\newcommand{\fmfldpLargeTenf}{0.5131}

\newcommand{\fmfldpLargeTwo}{0.2315}
\newcommand{\fmfldpLargeFive}{0.3980}
\newcommand{\fmfldpLargeTen}{0.5384}

\newcommand{\smsm}{0.1160} 
\newcommand{\sm}{0.1010} 
\newcommand{\smla}{0.1090} 

\newcommand{\ms}{0.6325} 
\newcommand{\mm}{0.1988} 
\newcommand{\ml}{0.1216}

\newcommand{\ls}{0.6972} 
\newcommand{\lm}{0.6801}
\newcommand{\lala}{0.5654}

\newcommand{\percentGainOne}{10\%}
\newcommand{\percentGainTwo}{5\%}

\newcommand{\smallToMedium}{150\%}
\newcommand{\mediumToLarge}{50\%}

\AtBeginDocument{%
  \providecommand\BibTeX{{%
    \normalfont B\kern-0.5em{\scshape i\kern-0.25em b}\kern-0.8em\TeX}}}

\copyrightyear{2021}
\acmYear{2021}
\setcopyright{acmlicensed}\acmConference[RecSys '21]{Fifteenth ACM Conference on Recommender Systems}{September 27-October 1, 2021}{Amsterdam, Netherlands}
\acmBooktitle{Fifteenth ACM Conference on Recommender Systems (RecSys '21), September 27-October 1, 2021, Amsterdam, Netherlands}
\acmPrice{15.00}
\acmDOI{10.1145/3460231.3474262}
\acmISBN{978-1-4503-8458-2/21/09}

\begin{document}

\title{Stronger Privacy for Federated Collaborative Filtering With Implicit Feedback}

\author{Lorenzo Minto}

\author{Moritz Haller}

\affiliation{%
   \institution{Brave Software}
   \city{London}
   \country{UK}}

\author{Hamed Haddadi}

\author{Benjamin Livshits}

\affiliation{%
    \institution{Brave Software}
   \city{London}
   \country{UK}
 }
\affiliation{%
    \institution{Imperial College London}
   \city{London}
   \country{UK}
 }

\renewcommand{\shortauthors}{Minto, et al.}

\begin{abstract}
   Recommender systems are commonly trained on centrally-collected user interaction data like views or clicks. This practice however raises serious privacy concerns regarding the recommender's collection and handling of potentially sensitive data. Several privacy-aware recommender systems have been proposed in recent literature, but comparatively little attention has been given to systems at the intersection of implicit feedback and privacy. To address this shortcoming, we propose a practical federated recommender system for implicit data under user-level local differential privacy (LDP). The privacy-utility trade-off is controlled by parameters $\epsilon$ and $k$, regulating the per-update privacy budget and the number of $\epsilon$-LDP gradient updates sent by each user, respectively. To further protect the user’s privacy, we introduce a proxy network to reduce the fingerprinting surface by anonymizing and shuffling the reports before forwarding them to the recommender. We empirically demonstrate the effectiveness of our framework on the MovieLens dataset, achieving up to \textit{Hit Ratio} with K=10 (HR@10) 0.68 on 50,000 users with 5,000 items. Even on the full dataset, we show that it is possible to achieve reasonable utility with HR@10>0.5 without compromising user privacy.
\end{abstract}





\maketitle

\section{Introduction}
\label{sec:intro}
For many years now, \textit{recommender systems} have been an integral part of the users’ experience: helping personalize software products around users’ particular tastes and needs. More specifically, \textit{recommender systems} are used to estimate the user’s preferences from a history of collected personal feedback, which can be either explicit, such as movie ratings, or implicit, which comprises more generic interaction data like views or clicks.

While such systems can significantly enhance user experience, they usually do so at the expense of privacy, by centrally collecting user data. 
Several studies~\cite{Shyong06, Calandrino2011, McSherry2009, Weinsberg2012, Narayanan2008, MacAonghusa2016} have focused on the privacy risks that come with the centralised approach. 
Some of these studies~\cite{Weinsberg2012, Narayanan2008, MacAonghusa2016} highlight how even seemingly non-sensitive data like movie ratings can be used to infer age, gender, and political affiliation of users. Moreover, Calandrino~\etal~\cite{Calandrino2011} have shown how \textit{recommender systems} can be easy targets of inference attacks, whereby the attacker manipulates the recommender to reveal personal information about other users. Finally, an inherent risk of any centralized system is the violation of privacy promises made by the data curator, either knowingly for illegitimate purposes, or through unintended data breaches\footnote{\url{https://en.wikipedia.org/wiki/List_of_data_breaches}}.

These concerns motivate the need for a privacy-first approach. More specifically, we suggest:
\begin{enumerate}
    \item No interaction data should be collected or stored by the recommender.
    \item User contributions should be protected by user-level differential privacy.
    \item Recommendation should be local and on-device.
\end{enumerate}
Current work on privacy-preserving recommendation frameworks either fails to provide formal privacy guarantees~\cite{ammaduddin19, Duriakova2019} or relies on direct or partial access to users' interaction data \cite{Hua, Gao2020}. Shin~\etal\cite{Shin2018} address all the above concerns but rely on a combination of dimensionality reduction through random projection and sparse recovery to achieve reasonable utility. Their method is only suitable when the unperturbed gradient data is sparse, as is the case for recommendation with explicit feedback (i.e. ratings).

We propose a method to achieve all of the above, while minimizing trust on behalf of the user and maintaining reasonable utility. We do so by enhancing existing, federated recommendation methods with a privacy mechanism to prevent reconstruction attacks from gradient data~\cite{Hitaj2017, WangGAN}. For each user, we control the privacy-utility trade-off by selecting a subset of $k$ factors from the item gradient update matrix and applying a privatization mechanism to each that ensures the perturbed factor meets $\epsilon$-local differential privacy. To reduce the fingerprinting surface and prevent a malicious recommender from building longitudinal representations of the user, we introduce a proxy network that breaks linkability between subsequent reports belonging to the same user by stripping metadata from said user’s traffic and shuffling it together with reports from other users.

\subsection{Contributions.}
In summary, in this work we make the following contributions:

\begin{enumerate}
  \item We propose a federated recommender system for implicit feedback working under user-level local differential privacy, where the privacy-utility trade-off is controlled by parameters $\epsilon$ and $k$, regulating the per-update privacy budget and the number of $\epsilon$-LDP gradient updates sent by each user, respectively.
  
  \item We experimentally demonstrate the effectiveness of our proposed framework by analysing the system's performance as a function of various communication and  privacy budgets, as well as item and user set cardinalities. On the MovieLens dataset we show that the system achieves HR@10 up to 0.68 on 50k users with 5k items, and even on the full dataset we show that it is possible to achieve reasonable utility with HR@10>0.5, without compromising user privacy.
  
  \item We further examine the communication costs associated with running the recommendation protocol, and present an analysis of its formal privacy guarantees. We further compare our system against several non-private and privacy-preserving benchmarks, showing that we can provide acceptable levels of utility, while offering stronger privacy guarantees.

\end{enumerate}

\subsection{Paper organization.}
The rest of this paper is organized as follows.
Section~\ref{sec:background} introduces the problem and notation of collaborative filtering algorithms for implicit feedback, and recalls the definition of local differential privacy. Section~\ref{sec:overview} gives an overview of the proposed system. Section~\ref{sec:eval} describes our experimental work and reports the results.
In Section~\ref{sec:discussion} we further discuss the proposed system by taking into consideration privacy and communication costs, and by comparing our solution with competing methods.
Finally, Section~\ref{sec:related} presents an overview of related work before concluding with Section~\ref{sec:conclusions}.

\section{Background}
\label{sec:background}

\subsection{Collaborative Filtering with Implicit Feedback}
\label{sec:implicit_cf}

Collaborative filtering is one of the most popularly used approaches in recommender systems. The goal of collaborative filtering is to model user preferences over a set of items by leveraging the "wisdom of the crowds". User behaviour is described by a set of \textit{interactions} $r_{ui}$, where $u$ is the user index and $i$ is the index of the item being interacted with. For implicit data---our case, $r_{ui}$ indicates user $u$ has interacted with item $i$, or how many times that interaction has happened. Implicit feedback is more abundant, as it does not require an active effort by the user, but it is also less informative than its explicit counterpart, as there is no clear way to infer the qualitative judgment made by the user, i.e. simply knowing a film has been watched does not tell us anything about the user's \textit{preference} for that film.

The first formulation of collaborative filtering for implicit feedback, relying on Singular Value Decomposition (SVD) of the user-item interaction matrix, is proposed in~\cite{Hu2008}. The authors adapt the matrix factorization methods already available for explicit data to the implicit case. First, they define \textit{preference} $p_{ui}$ as a binary indicator $\{0,1\}$ of whether user $i$ has a preference for item $i$, and assign a \textit{confidence} $c_{ui}$ score to this preference. They choose to model $c_{ui}$ as the following:
$$ c_{ui} = 1 + \alpha r_{ui} $$
where $r_{ui}$ is the number of times user $u$ has interacted with item $i$. This definition follows the intuition that an item that has been interacted with multiple times is more likely to be preferred by the user. Second, the cost function is modified from the explicit case to account for all possible $u,i$ pairs, and to accommodate for the varying confidence levels in each observation. The implicit loss function is defined as the following:
\begin{align}
L(\theta) = \sum_{u,i} c_{ui}(p_{ui}-\mathbf{x}_u^T\mathbf{y}_i)^2 + \lambda \left( \sum_{u} ||\mathbf{x}_u||^2 + \sum_{i} ||\mathbf{y}_i||^2 \right)
\end{align}
with $\lambda$ being a joint regularization parameter. The partial derivatives with respect to $\mathbf{x}_u$ and $\mathbf{y}_i$ are given by:
\begin{align}
    \dfrac{\partial J}{\partial \mathbf{x}_u} = -2\sum_i [c_{ui}(p_{ui}-\mathbf{x}_u^T\mathbf{y}_i)]\mathbf{y}_i + 2\lambda\mathbf{x}_u \\ 
    \dfrac{\partial J}{\partial \mathbf{y}_i} = -2\sum_u [c_{ui}(p_{ui}-\mathbf{x}_u^T\mathbf{y}_i)]\mathbf{x}_u + 2\lambda\mathbf{y}_i
\end{align}

\subsection{Federated Collaborative Filtering}
Federated Learning~\cite{McFed, McFed2} has been a fast growing field in machine learning research. It was proposed originally as a way to train a central model on privacy-sensitive data distributed across users' devices. In the federated learning paradigm, the user's data never has to leave the client. Instead, clients train a local model on their private data and share model updates with the server. These updates are then aggregated (typically by averaging) and a global model update is performed. Finally, the updated model is sent back to each client and the process is repeated, until convergence or until satisfying performance is achieved. 

In a recently proposed method for federated collaborative filtering~\cite{ammaduddin19}, the authors distribute the implicit matrix factorization problem introduced in the previous section between a server and multiple clients. At the very beginning, the server randomly initializes an item embeddings matrix $V$ and clients initialize their own user embeddings $\textbf{x}_u$ locally. The item matrix is then shipped to the clients, who locally compute an updated user vector in closed form
\begin{equation}
  \label{eqn:user_update}
  \textbf{x}_u^* = (VC^uV^T + \lambda I)^{-1}VC^u\textbf{p}(u)  
\end{equation}
Where $C^{u}$ is an $M\times M$ diagonal matrix with $C^{u}_{ii}=c_{ui}$ and $\textbf{p}(u)$ is a vector that contains the user preferences, and a client-based gradient update for the item matrix
\begin{equation}
    \label{eqn:item_user_update}
    f(u,i) = [ c_{ui}(p_{ui}-\textbf{x}_u^T\textbf{v}_i) ]\textbf{x}_u
\end{equation}
 Gradient updates from all clients are aggregated together at the server and they are used to update the central item matrix with the following
\begin{align}
    \dfrac{\partial J}{\partial \textbf{v}_i} =& -2\sum_u f(u,i) + 2\lambda \textbf{v}_i \\
    \label{eqn:item_update}
    \textbf{v}_i &= \textbf{v}_i - \gamma \dfrac{\partial J}{\partial \textbf{v}_i} 
\end{align}
The updated item matrix is then sent down to all the clients and the federated process repeats.

\subsection{Local Differential Privacy}

Differential privacy~\cite{dworkDP, Dwork2013} has become the gold standard for strong privacy protection. It provides a formal guarantee that a model's results are negligibly affected by the participation or less of any \textit{single} individual. Differential privacy was originally formulated in a \textit{central} setting: a trusted aggregator collects raw user data and injects controlled noise either in the query inputs, outputs, or both. \textit{Local} differential privacy, on the other hand, is a particular case of the differential privacy framework where the aggregator is not trusted and the noise injection is done directly by the client. Since no trust in the aggregator is required, the local version of DP can provide much stronger privacy protection. Companies like Google and Apple have applied the local model of differential privacy to their privacy-preserving data analytics protocols~\cite{Erlingsson2014, sketch_apple, Zhu2020} and more recently to federated learning~\cite{Bhowmick2018}. In order to define local differential privacy, we first define a local randomizer:

\begin{definition}
\cite{Dwork2013} A randomized algorithm $\Phi: \mathcal{U}\to \mathcal{Y}$ is an $\epsilon$-local randomizer, where $\epsilon>0$, if for all input pairs $x$ and $x'$ in D, and any output Y of $\Phi$, we have
$$ \text{Pr}[\Phi(x) = Y] \leq e^\epsilon \cdot \text{Pr}[\Phi(x') = Y],$$
\end{definition}
Where $x$ and $x'$ belong to universe $\mathcal{U}$ of user data. Now that we have defined a local randomizer, we can proceed to give the definition of local differential privacy:

\begin{definition}
~\cite{Dwork2013, sketch_apple} Let $\Phi: \mathcal{U}^n \to \mathcal{Z}$ be a randomized algorithm mapping a dataset with n records to some range $\mathcal{Z}$. Algorithm $\Phi$ is $\epsilon$-local differentially private if it can be written as $\Phi(d^{(1)}, ..., d^{(n)}) = f(\Phi_1(d^{(1)}), ..., \Phi_n(d^{(n)}))$, where each $\Phi_i : \mathcal{U}\to \mathcal{Y}$ is an $\epsilon$-local randomizer, and $f: \mathcal{Y}^n \to \mathcal{Z}$ is some post-processing function of the privatized records $\Phi_1(d^{(1)}),...,\Phi_n(d^{(n)})$.
\end{definition}

The privacy budget $\epsilon$ controls the trade-off between utility and privacy: when $\epsilon=0$, we have perfect privacy and no utility, while for $\epsilon=\infty$ we would have no privacy and perfect utility.

\section{System}
\label{sec:overview}

\subsection{System Overview}
We propose a novel framework to solve the recommendation task under the constraints set out in Section I by integrating federated recommendation~\cite{ammaduddin19, Lin2020} and local differentially private protection mechanisms~\cite{Duchi, Nguyen2016}. Figure~\ref{fig:system_design} illustrates the three components of our proposed system, which we describe in detail below.

\begin{figure*}[t]
  \centering
  \includegraphics[keepaspectratio, width=0.6\textwidth]{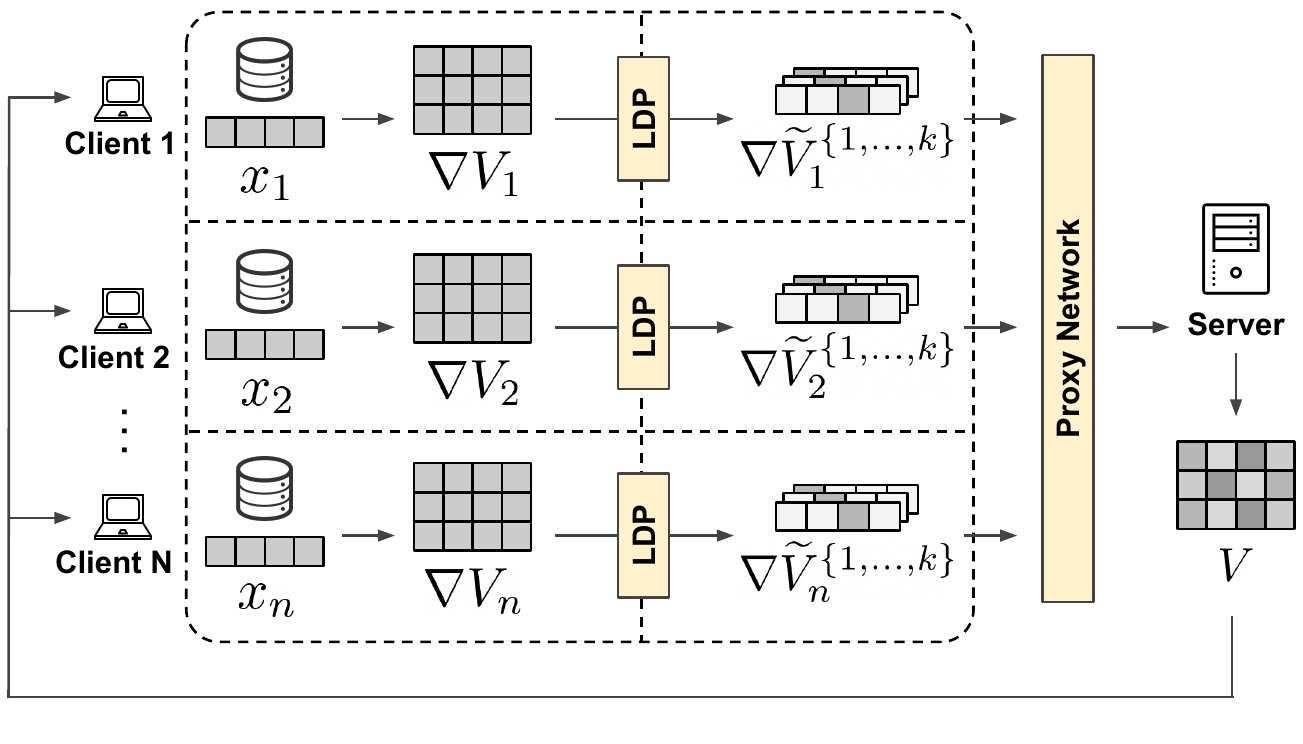}
  \caption{The proposed system design of privacy-preserving federated collaborative filtering with implicit feedback. The user vector $x$ never leaves the device.}
  \Description[Figure 1]{Fully described in the text.}
  \label{fig:system_design}
\end{figure*}

\subsubsection{Client with LDP Module}
Initially, each client randomly initializes an $F$-dimensional local user embedding $\textbf{x}_u$. This user embedding \textit{never} leaves the client. Each participant receives user-agnostic item matrix $V$ from the server, signalling the start of a federated epoch. The local user embedding and the item matrix are used to compute 
(1) an updated user embedding $\textbf{x}_u^*$, which is, again, kept strictly private, and 
(2) a gradient update $\nabla V_u$ for item matrix $V$. Item-matrix gradient update $\nabla V_u$ is passed to the \textit{LDP Module}, where we pick $k$ factors uniformly at random and to each apply our privatization mechanism with per-report privacy budget $\epsilon$. 
These $k$ $\epsilon$-LDP reports are then sent to a proxy network. On arrival of the new item matrix $V^\prime$, a new federated epoch starts and the above process repeats. 

Once the training is finished, the user embedding $\textbf{x}_u$ can be used in conjunction with the latest item matrix $V$ to compute on-device the confidence that the user will interact with item $i$ by taking the dot product between the two relevant embeddings $\langle \textbf{x}_u , \textbf{v}_i \rangle$.

\textbf{Input.} Item matrix $V,$ per-update privacy budget $\epsilon$, number of updates~$k$.

\textbf{Output.} $k$ $\epsilon$-LDP gradient reports.

\subsubsection{Proxy Network}
The proxy network sits between the server and the client. Once the messages containing the $k$ $\epsilon$-LDP reports reach the proxy network, they are stripped of their metadata (such as IP address), split into the single reports, shuffled with the reports from other users' messages to break any existing timing patterns, and forwarded to the server. The proxy network takes care of breaking linkability between the streams of $k$ reports coming from each client at each epoch. This greatly reduces the user fingerprinting surface and prevents the recommender from building any longitudinal representation of the user, both within and across epochs.

\textbf{Input.} $k$ $\epsilon$-LDP gradient reports from each of $N$ users.

\textbf{Output.} $N\times k$ unlinkable $\epsilon$-LDP gradient reports.

\subsubsection{Server}
The server randomly initializes an $M\times F$ item matrix, which constitutes the global part of the shared model. This is shipped to all the clients to initiate the federated learning process in epoch 0. At each epoch, once enough $\epsilon$-LDP gradient reports reach the server, they are aggregated together to reconstruct a global item matrix update $\nabla V$. This is used to compute an updated item matrix $V^\prime$, which is then shipped to each client to initiate the next federated learning epoch. 

\textbf{Input.} $N\times k$ unlinkable gradient reports.

\textbf{Output.} Updated item matrix $V^\prime$.

\subsection{Proposed Method}
For our study, we adopt the matrix factorization-based federated collaborative filtering framework as proposed in ~\cite{ammaduddin19}. At the start of each learning epoch, each participant in the recommendation network receives the most recent item matrix $V$ from the server. Each client computes an updated $F$-dimensional user embedding $\textbf{x}_u^*$ using (\ref{eqn:user_update}), and an $MF$-dimensional item matrix gradient update $\nabla V_u$ using (\ref{eqn:item_update}). To privatize the gradient updates before shipping them to the server we adapt to our recommendation task the randomized binary response mechanism proposed in~\cite{Nguyen2016}. This method allows computing aggregates of users' gradients while satisfying local differential privacy. In particular, the proposed method takes in a numerical, multidimensional gradient matrix, randomly selects one item and one factor from it and encodes the selected value in the transmission frequency of two opposite global constants $\{ B, -B \}$, where $B$ is defined as follows
\begin{equation}
    B = \dfrac{e^\epsilon+1}{e^\epsilon-1}\cdot MF
\end{equation}
with B only depending on the input data dimensionality $M\cdot F$ (same for all users) and per-update privacy budget $\epsilon$ (shared between all users).
Algorithm (\ref{alg:ldp_perturbation}) illustrates the adopted perturbation mechanism. 

% \par\medskip\noindent
% \centerline{\begin{minipage}[c]{0.7\textwidth}
%     \begin{algorithm}[H]
%     \SetAlgoLined
%     \KwIn{Item gradient matrix $\nabla V_u$}
%     %\KwOut{Item gradient matrix  $\nabla \Tilde{V}_u$ satisfying $\epsilon$-LDP}
%      Initialize $\nabla \Tilde{V}_u = 0^{M\times F}$ \\
%      Sample item $i$ uniformly at random\\
%      Project $\nabla V_u^i$ onto $[-1,1]$\\
%      Sample factor $f$ uniformly at random\\
%      Sample a Bernoulli variable $\beta$ such that
%      $$ Pr[\beta=1] = \dfrac{\nabla V_u^{i,f}\cdot(e^\epsilon -1)+e^\epsilon+1}{2e^\epsilon+2} $$
%      \nl \eIf{$\beta = 1$}{
%      $\nabla \tilde{V}_u^{i,f} = \dfrac{e^\epsilon+1}{e^\epsilon-1}\cdot MF$
%      }{
%      $\nabla \tilde{V}_u^{i,f} = - \dfrac{e^\epsilon+1}{e^\epsilon-1}\cdot MF$
%      }
%      \Return $\nabla \Tilde{V}_u$
%      \caption{Gradient LDP perturbation mechanism, adapted from \cite{Nguyen2016}}
%      \label{alg:ldp_perturbation}
%     \end{algorithm}
% \end{minipage}}
% \par\medskip\noindent

\begin{algorithm}
\SetAlgoLined
\KwIn{Item gradient matrix $\nabla V_u$}
%\KwOut{Item gradient matrix  $\nabla \Tilde{V}_u$ satisfying $\epsilon$-LDP}
 Initialize $\nabla \Tilde{V}_u = 0^{M\times F}$ \\
 Sample item $i$ uniformly at random\\
 Project $\nabla V_u^i$ onto $[-1,1]$\\
 Sample factor $f$ uniformly at random\\
 Sample a Bernoulli variable $\beta$ such that
 $$ Pr[\beta=1] = \dfrac{\nabla V_u^{i,f}\cdot(e^\epsilon -1)+e^\epsilon+1}{2e^\epsilon+2} $$
 \nl \eIf{$\beta = 1$}{
 $\nabla \tilde{V}_u^{i,f} = \dfrac{e^\epsilon+1}{e^\epsilon-1}\cdot MF$
 }{
 $\nabla \tilde{V}_u^{i,f} = - \dfrac{e^\epsilon+1}{e^\epsilon-1}\cdot MF$
 }
 \Return $\nabla \Tilde{V}_u$
 \caption{Gradient LDP perturbation mechanism, adapted from \cite{Nguyen2016}}
 \label{alg:ldp_perturbation}
\end{algorithm}

The output of this mechanism $\nabla \Tilde{V}$ can be represented by a tuple $(i,j)$ where $i$ is the index of the selected gradient factor, and $j$ is either 0 or 1, for whether $B$ or $-B$ has been selected. This message is sent to the recommender, where it's aggregated with those from other users to form a global item matrix gradient update. More in line with the conventional federated learning literature and differing from \cite{ammaduddin19}, we average over the collected gradients.
\begin{equation}
    \nabla \Tilde{V} = \dfrac{1}{N} \sum_u^N  \nabla \Tilde{V}_u  
    \label{eqn:estimator}
\end{equation}
Nguyên~\etal\cite{Nguyen2016} show that the mean of all users' raw values can be estimated by taking the average over the perturbed values, which in our adaptation translates to (\ref{eqn:estimator}) being an unbiased estimator of $\nabla V$. However, from the experiments run on this setup we observe that our federated protocol takes long to learn and ultimately performs poorly. The information exchanged with the server at each epoch is extremely sparse and very little learning can be extrapolated from it. More specifically, each user only contributes to only as much as $\frac{1}{MF}$ of the gradient matrix, with reports being further randomized. This ultimately leads to the computed aggregates being very poor estimates of the target values, even with a large number of users participating in the protocol.  

In order to improve the accuracy of our system, we increase the amount of information shared by each user at each epoch by setting the number of randomly elected factors to $k$. Each gradient factor gets perturbed separately resulting in $k$ tuples $(i, j)$ being produced by each client and subsequently being shared with the server. This procedure, however, increases the privacy budget by a factor of $k$ from the composition theorem~\cite{Dwork2013}. If with $k=1$ our system satisfied $\epsilon$-LDP, now it satisfies $k\epsilon$-LDP, where $k$ is the number of tuples exchanged by each user with the server. The $k$ $\epsilon$-LDP reports are sent to the proxy network, which then forwards them to the server where the global item matrix gradient update is estimated as (\ref{eqn:estimator}). The full learning protocol is shown in Algorithm \ref{alg:system_algorithm}.

\begin{algorithm}
\SetAlgoLined
 Initialize $X$, $V$ \\
 \For{$t \in \{ 0,1,2...T \}$}{
    \For{$u \in \{ 0,1,2...N \}$}{
        Derive $\textbf{x}_u^*$ with (\ref{eqn:user_update}) 
        \Comment*[r]{$\forall$ clients}
        Compute $\nabla V_u$ with (\ref{eqn:item_user_update}) \\
        \For{$i \in 0,1,2...K$}{
            $\nabla \tilde{V}_u^{\{i\}} = \text{Algorithm1}(\nabla V_u, \epsilon)$
        }
    }
    Collect $\nabla \tilde{V}_u^{\{1,...,k\}}$ from all clients  \Comment*[r]{server}
    $\nabla \tilde{V} = \dfrac{1}{N} \sum_u^N \nabla \tilde{V}_u$ \\
    $V = V - \gamma \{ \nabla \tilde{V} + 2\lambda V \}$ \\
 }
 \Return $X$, $V$
 \caption{Local differentially private federated recommender system}
 \label{alg:system_algorithm}
\end{algorithm}

\section{Experiments}
\label{sec:eval}

\subsection{Data}
Experiments are run on the MovieLens \cite{MovieLens} dataset. The data consists of ratings from ~138,493 users on ~27,278 movies. Ratings range from~0.5 to~5 in~0.5 increments. To convert the~20M interactions into implicit feedback, we transform each rating to a binary value, with the following rule:
\begin{equation}
  r_{ui}^\text{imp} =
    \begin{cases}
      1, & \text{if $r_{ui}^\text{exp}>0$}\\
      0, & \text{otherwise}
    \end{cases}       
\end{equation}
In short, $r_{ui}$ is~1 if user $u$ has watched item $i$,~0 otherwise. Analogous to \cite{Gao2020}, we retain only users and items with at least~60 interactions. Our final dataset is~97.7\% sparse with \numOfInteractionsFinal{} interactions, from \numOfUsersFinal{} users on \numOfItemsFinal{} films. To further study the impact of user and item set cardinality on system performance we sample nine subsets from the full data set, with item set cardinalities \textit{small} (1,000), \textit{medium} (5,000), \textit{large} (9,781) and user set cardinalities \textit{small} (1,000), \textit{medium} (10,000), \textit{large} (50,000).

\subsection{Evaluation}
For evaluation we adopt the commonly used leave-one-out technique~\cite{Gao2020, He2017}. For each user we randomly sample one interaction and use it as the test item. We cross-validate the model by running multiple random leave-one-out splits and presenting the mean results. To measure the performance of our recommender system, we adopt the popular \textit{Hit Ratio} (HR@$K$) metric, where $K$ is the number of recommended items. HR@$K$ measures the probability of the recommender ranking the left-out test item in its top-$K$ recommendations. To establish comparability between data from different items spaces, we follow recent literature~\cite{He2017, Gao2020} by sampling~99 items that the user hasn't interacted with, appending the test item, and ranking the resulting subset. 

\subsection{Parameters}
As parameters we chose $f=5$ factors, a regularization term of $\lambda=10^{-6}$, and a learning rate of $\gamma=10^{-3}$, though values may vary slightly between experiments. The parameters were optimised with grid search over multiple cross-validation runs. Each epoch of training performs~20 gradient descent steps to update the global item matrix. We consider different per-update privacy budgets and numbers of updates, $\epsilon$ and $k$, respectively.

\subsection{Results}
In the following we outline the results of our experiments. We first study the privacy-utility trade-off as determined by $\epsilon$ and $k$. We then look at learning curves, as well as the impact of item set cardinality on utility. We end by comparing our method against various private and non-private benchmarks.

\subsubsection{Privacy-utility trade-off} Figure 2 shows the system utility as measured by HR@10 for the small item set, all three user set sizes and varying per-update privacy budget $\epsilon$ and number of updates $k$. As expected, utility is on par with the random baseline for $k=1$ across all user set sizes, as the signal sent to the recommender is insufficient to guide any learning. However, with increasing user set size, as well as per-report privacy budget and number of updates, we observe a consistent and significant increase in utility---in some cases reaching HR@10 of 0.7 and higher, while maintaining strong privacy---demonstrating the efficacy of the proposed system. When analysing the rate of improvement, it is apparent that increasing the number of updates $k$ and increasing the number of users, both yield diminishing returns. E.g. we observe a $\approx$\smallToMedium{} increase in utility from small to medium user set size and only a $\approx$\mediumToLarge{} increase from medium to large. Similarly, increasing $k$ from 50 to 100 results in a \percentGainOne{} utility improvement over merely \percentGainTwo{} when increasing $k$ from 100 to 250. It should be noted that increasing the number of updates $k$, unlike increasing user set size, has a direct impact on communication as well as privacy costs, so it should always be justified in terms of utility gain.

\begin{figure*}[tb]
  \centering
  \includegraphics[keepaspectratio, width=0.295\textwidth]{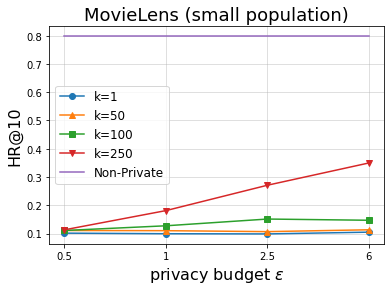}
  \includegraphics[keepaspectratio, width=0.295\textwidth]{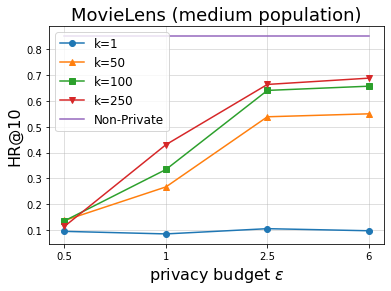}
  \includegraphics[keepaspectratio, width=0.295\textwidth]{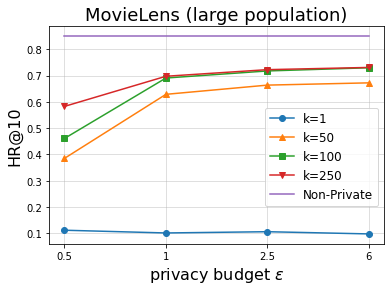}
  \caption{HR@10 performance with varying privacy budget $\epsilon$ and \textit{small, medium} and \textit{large} user set sizes. Number of epochs has been fixed to~20 using the \textit{small} item set.}
  \label{fig:HR_eps}
  \Description[HR@10 for small, medium and large user set sizes, over multiple privacy budgets and number of updates]{We consider privacy budgets epsilon 0.5, 1, 2.5, 6. The first subplot shows hit rate at 10 for the small user set. Accuracy is en par with the random baseline for k equals to 1 and 50 across all the privacy budgets, increasing only slightly with k equals 100 to hit rate at 10 equals 0.15. For k equals 250, accuracy has significantly improved to hit rate at 10 equals 0.35. The second subplot shows accuracy for the medium user set. Accuracy for k equal to 1 is still en par with random, and so is accuracy for all considered ks and epsilon equals 0.5. For k equal 50, 100, 250 and epsilon greater than 0.5 accuracy increases consistently across the adopted privacy budgets and reaches its peak at epsilon equals 6 and k equals 250 at hit rate at 10 slightly below 0.7. In the last subplot, accuracy for k equals 1 is still en par with random for all epsilons. For k equals 50, 100, 250 accuracy steadily increases for all epsilons. Already for epsilon equals 0.5 and k equals 50 we have hit rate at 10 around 0.4. For the same epsilon and k equals 250 we have hit rate at 10 around 0.6. Accuracy peaks at hit rate at 10 above 0.7 for epsilon equals 6 and k equals 250.}
\end{figure*}

\subsubsection{Learning curves} Figure~\ref{fig:HR_epochs} shows HR@10 against the number of epochs on the validation set for a small item set with medium and large user set sizes and two different communication budgets of $k$=50 and $k$=100. The privacy budget is fixed at $\epsilon$=2.5. Unsurprisingly, we can observe that the number of exchanged updates $k$ has significant impact on the rate of learning as a higher quality signal reaches the recommender at each epoch. Reasonable utility of HR@10>0.5 can already be achieved after only 5-10 epochs, which is important for producing quality recommendations quickly, as well as keeping apace with user interest drift. Finally, user set size seems to affect only asymptotic performance, with larger user sets showing higher performance at convergence, corroborating findings from Figure~\ref{fig:HR_eps}.

\begin{figure}[tb]
  \centering
  \includegraphics[keepaspectratio, width=0.42\textwidth]{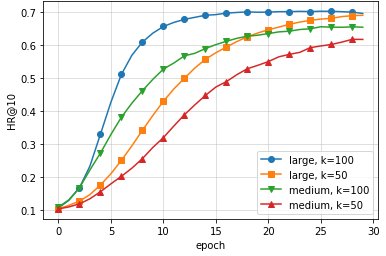}
  \caption{Learning curves with HR@10 validation set performance across 30 epochs, for varying $k$ and user set sizes.}
  \label{fig:HR_epochs}
  \Description[Plot of HR@10 across learning epochs for k=50 and k=100 for medium and large user sets]{ Learning curves all start at epoch 0 and random performance hit rate at 10 equals 0.1. After epoch 10, for k equals 100 we have the hit rate at 10 is slightly above 0.65 for the system trained on the large user set and hit rate at 10 is slightly below 0.55 for system trained on the medium user set. For k equal 50, hit rate at 10 on the large user set is around 0.42 and hit rate at 10 on the medium user set is slightly above 0.3. After epoch 30, the order of the curves has slightly changed as the system trained on the large user set with k equal 50 has outperformed the system trained on the medium user set with k=100. The systems trained on the large user set have both converged to around hit rate at 10 equals 0.7. While, the systems trained on the medium user set have reached hit rate at 10 equals 0.65 with k equal 100, and hit rate at 10 slightly above 0.6 with k equal 50.}
\end{figure}

\subsubsection{Impact of item set cardinality on performance} We extend our discussion to different item set sizes and in Figure~\ref{fig:pop_vs_items_table} report HR@10 for various combinations of user and item set sizes. Users only report a small subset of their local gradient. As the number of items increases by $\Delta M$, the local gradient dimensionality increases by $\Delta M \times F$, where $F$ is the number of factors. As a consequence, while maintaining the number of updates $k$ fixed, the signal sent to the recommender in each epoch becomes sparser and loses in quality, thus negatively impacting the overall performance of the system. Empirically, we observe that utility decreases with increasing item set size. The effect is more pronounced for smaller user set sizes, while for the large user set size, utility degrades slower, showing acceptable performance (HR@10>0.5) across all item set sizes.

\begin{figure}[tb]
    \centering
    \setlength{\tabcolsep}{16pt}
    \begin{tabular}{ cccc}
     \toprule
     Users/Items & 1K & 5K & 10K \\ 
     \midrule
     1,000 & \smsm & \sm & \smla  \\
     10,000 & \ms & \mm & \ml \\
     50,000 & \bf \ls & \lm & \lala \\
     \bottomrule
    \end{tabular}
    \caption{HR@$10$ performance for datasets of varying user set and item set cardinalities. Number of epochs fixed at~20, $\epsilon$ to~2.5 and number of updates $k=100$.}
    \label{fig:pop_vs_items_table}
\end{figure}

\subsubsection{Benchmark against private and non-private benchmarks} The final experiment aims to compare the proposed system against various private and non-private benchmarks on the full dataset with number of updates $k$=100 (\textit{FMF-LDP-100}) and $k$=250 (\textit{FMF-LDP-250}) and fixed privacy budget of $\epsilon$=2.5. We chose a random baseline model to establish the lower bound and traditional non-private matrix factorizaion (\textit{MF-NP}) as a natural upper bound. We further compare our system with Gao~\etal\cite{Gao2020} \textit{DPLCF} which, to the best of our knowledge, is the most comparable method in existing literature. We replicate their validation split by using only the latest interaction as the test item for each user. Models are trained over 20 epochs and the utility is measured as top-$K$ hit rate for $K=\{2,5,10\}$. The results are shown in Figure~\ref{tab:topk_table}.
Our system clearly out-performs the random baseline with HR@10 of 0.5 and higher, corresponding to 5$\times$ improvement, showing that it is possible to provide reasonable utility without compromising user privacy. As expected, the system does not match performance of its non-private equivalent. On the full dataset the system also falls short of Gao~\etal\cite{Gao2020}, however we argue that our method is in fact a competitive alternative on data with smaller item set cardinality than the full MovieLense dataset as shown in Figure 2. In section~\ref{sec:privacy} we go on to show that our intuitive and formal privacy guarantees are significantly stronger.

\begin{figure}[tb]
    \centering
    \setlength{\tabcolsep}{12.5pt}
    \begin{tabular}{ lccc }
     \toprule
     Method & HR@2 & HR@5 & HR@10 \\ 
     \midrule
     MF-NP  & \mfnpTwo & \mfnpFive & \mfnpTen \\
     %FMF-NP & \fmfnpTwo & \fmfnpFive & \fmfnpTen \\
     DPLCF \cite{Gao2020} & \dplcfTwo & \dplcfFive & \dplcfTen \\
     FMF-LDP-100 & \fmfldpLargeTwof & \fmfldpLargeFivef & \fmfldpLargeTenf \\
     FMF-LDP-250 & \bf \fmfldpLargeTwo & \bf \fmfldpLargeFive & \bf \fmfldpLargeTen \\
     Random & 0.0200 & 0.0500 & 0.1000 \\
     \bottomrule
    \end{tabular}
    \caption{HR@$K$ for various benchmarks on the latest split of the full dataset after at most~20 epochs of training and with $\epsilon=2.5$, where applicable.}
    \label{tab:topk_table}
\end{figure}

\section{Discussion}
\label{sec:discussion}
In the following, we will discuss the privacy guarantees of our proposed system and show that it provides stronger privacy compared to similar methods. We will end with a practical analysis of the communication cost of running the protocol.

\subsection{Privacy Analysis}
\label{sec:privacy}
This section analyzes how our system ensures user-level local differential privacy. We will first show that individual gradient updates satisfy event-level $\epsilon$-local differential privacy, yielding $k\epsilon$-local differential privacy at user-level. Then we will argue how the proxy network can further safeguard privacy by establishing unlinkability of subsequent and concurrent gradient updates.

\begin{theorem}
Algorithm~\ref{alg:ldp_perturbation} satisfies $\epsilon$-local differential privacy.
\end{theorem}

\begin{proof}
To show that Algorithm \ref{alg:ldp_perturbation} is $\epsilon$-local differentially private, for all inputs $\nabla V_q, \nabla V_q^\prime$ and all outputs $v_q$, we have
\begin{align*}
    \dfrac{P(\Tilde{\nabla V_q}=v_q|\nabla V_q)}{P(\Tilde{\nabla V_q}=v_q|\nabla V_q^\prime)} &\leq \max_{v_q} \dfrac{P(\Tilde{\nabla V_q}=v_q|\nabla V_q)}{P(\Tilde{\nabla V_q}=v_q|\nabla V_q^\prime)} \\
    &= \dfrac{\max_{v_q} P(\Tilde{\nabla V_q}=v_q|\nabla V_q)}{\min_{v_q} P(\Tilde{\nabla V_q}=v_q|\nabla V_q^\prime)}
\end{align*}
where $\Tilde{\nabla V_q}$, as per Algorithm 1, is non-zero only at index $j$. Further assuming that the randomly chosen constant is positive, i.e. $\beta=1$, we have
\begin{align*}
    \dfrac{\max_{v_q^j} P(\beta=1|\nabla V_q^j)}{\min_{v_q^j} P(\beta=1|\nabla V_q^{j,\prime})} &= \dfrac{\max\limits_{\nabla V_q^{j} \in [-1,1]} \nabla V_q^j (e^\epsilon-1)+e^\epsilon+1}{\min\limits_{\nabla V_q^{j,\prime}  \in [-1,1]} \nabla V_q^{j,\prime} (e^\epsilon-1)+e^\epsilon+1} \\
    &= e^\epsilon
\end{align*}
The same holds if the randomly chosen constant is negative, i.e. ($\beta=0$).
\end{proof}

Having shown that individual gradient updates are $\epsilon$-local differentially private at event-level, we can easily see that at user-level we have $k\epsilon$-local differential privacy for $k$ updates via the composition theorem. This guarantee extends by Definition~2 to the aggregated updates across clients in Algorithm \ref{alg:system_algorithm}, showing that the system ensures $k\epsilon$-local differential privacy at user level.

The selected range of per-update $\epsilon$ in our experiments is common to related work in industry, such as \cite{sketch_apple}. It is also worth noting that the choice of $\epsilon$ needs to be calibrated to the type of data being protected. The same privacy budget applied to different types of data such as raw feedback, model parameters or model parameter updates, implies profoundly different privacy guarantees.

We now consider the privacy advantages of using a proxy network. First, the LDP-privatized reports are anonymized by the proxy network. This is done by removing metadata and shuffling the reports with reports from other users before forwarding them to the server. The server is thus unable to link together multiple reports from any individual client. This significantly reduces the fingerprinting surface on each user's gradient collection and prevents the recommender from building longitudinal profiles of the user, both within, as well as across epochs. Finally, Erlingsson~\etal\cite{erlingsson2020amplification} show that shuffling anonymized reports can significantly amplify privacy guarantees when viewed from the central model.

We conclude by illustrating how relatively little data the client has to send to the recommender. Assuming a round of~20 epochs and with $k=100$ updates each, a single user contributes to at most $20\times 100=2,000$ factors of the shared item gradient matrix. Using the full-size MovieLens as an example and setting~5 as the number of factors, a single user only sends at most $4\%$ of the entire item gradient matrix with $M\times D \approx 50,000$.

\subsection{Comparison with Similar Methods}
\label{sec:comparison}
Gao~\etal\cite{Gao2020} also focus on the problem of privacy-preserving top-$n$ recommendation with implicit feedback. In their proposed method, each user perturbs their interaction data before sending it to the recommender, which in return estimates item-to-item similarity based on the perturbed feedback. The learned item-matrix is then disseminated back to the users for local on-device recommendations based on past interaction history using kNN. Although we share similar motivations, their method differs significantly from ours.

As shown in Section~\ref{sec:eval}, our framework's performance for $k=250$ is $22\%$ below the performance of the privacy-preserving system proposed in~\cite{Gao2020}. However, our framework can still provide reasonably strong performance HR@10=0.5384, while providing significantly stronger privacy protection to the user. Gao~\etal\cite{Gao2020} privatization algorithm is applied on interaction data (i.e. how many times a film has been watched) and satisfies user-level $M\epsilon^\prime$-differential privacy, where $M$ is the item set size. Whereas, our privatization algorithm is applied on gradient data and satisfies $k\epsilon$-differential privacy, where $k$ is the number of updates. Since the two privatization mechanism act on different types of data, it is hard, if not impossible, to formally compare the respective budgets and guarantees. Intuitively, however, gradient data can be considered less privacy-sensitive than raw interaction data, meaning that, for a shared privacy budget $\epsilon$, a privacy guarantee on gradient data is stronger than one on interaction data. Even in the scenario where we assume these two privacy budgets to be comparable, since by design $k \ll M$, we have that our framework's user-level privacy budget is significantly smaller than the one of the competing method.

\subsection{Communication Cost}
Downstream, at each epoch of federated training each client receives an $M \times F$ real-valued item matrix $V$ from the server, where $F$ we consider fixed at 5. Depending on the size of the chosen item set, this message will have different sizes: in the best case, where number of items is~1,000, we have that $V$ weighs about~20 KB; whereas, when using the full~9,781 item set we have that $V$ weighs about~0.2 MB. Assuming a maximum of 20 epochs of training and that the user participates to all of them, the total worst-case download cost for a single run of our federated system is~4 MB. Upstream, each client transmits $k$ tuples $(i, j)$, where $i$ and $j$ are respectively one~4-byte integer and one-bit boolean. At each epoch, each client uploads about~$4k$ bytes, where $k$ is the number of tuples. Assuming $k=100$, which is the most common communication budget in our analysis, the uploaded data per epoch by each user would amount to~$0.4$ KB. For a total of~20 epochs, the total upload cost is~8 KB.

\section{Related Work}
\label{sec:related}

\subsection{Privacy-Preserving Recommendation}
In this review, we only consider work that treats the recommender as a potential attacker. We also define three levels of user data granularity: raw, model weights, gradient updates. Sharing raw data has the greatest privacy risk, while sharing gradient updates the lowest. Under the defined setting, we identify two main categories of methods.  

\subsubsection{Data Collection}
In this category, a protection mechanism is adopted to privatize user \textit{raw} data before sharing it with the recommender. These methods rely on a central model and data collection, however obfuscated, happens at \textit{raw} level. Polat~\etal\cite{Polat2005} inject Gaussian noise in the reported ratings and use SVD-based collaborative filtering to produce privacy-preserving recommendations, however the authors fail to provide any formal privacy guarantees for their method. Li~\etal\cite{Li2017} study the problem of point-of-interest recommendation and propose an ad-hoc solution that involves transforming the raw trajectories into bipartite graphs before injecting carefully calibrated noise to meet $\epsilon$-differential privacy guarantees. Shen and Jin~\cite{Shen1, Shen2} propose privacy-preserving recommendation systems where user ratings are perturbed locally with provable privacy guarantees before being submitted to the recommender. More closely related to our specific problem setting, Gao~\etal\cite{Gao2020} adopt a differentially private protection mechanism based on randomized response to obfuscate private \textit{raw} level interaction data before reporting it to the server. This data is used to estimate an item similarity matrix which is sent to users, who can then produce the recommendation results locally, further safeguarding their privacy.

All the above methods guarantee $\epsilon$-differential privacy at \textit{event-level}, that is they can only reduce the impact on the results of a certain interaction~---~a rating or click. Our framework, instead, strives to guarantee \textit{user-level} differential privacy in order to protect the entirety of a user's interactions. It should be noted that an $\epsilon$-DP guarantee at event-level naively translates, by the composition theorem~\cite{Dwork2013}, to a $d\epsilon$-DP guarantee at user-level, where $d$ is the dimensionality of the transmitted data and $\epsilon$ the event-level privacy budget.

\subsubsection{Distributed/Federated} In this category, learning is distributed between clients and server. Users' interaction data is kept strictly local and only gradient updates of a local model are shared with the server. As raw gradients have been shown to leak user information~\cite{WangGAN, Hitaj2017}, gradient data is passed through a privatization mechanism and only then sent to the server. Hua~\etal~\cite{Hua} inject Gaussian noise into the local gradient updates to ensure that the transmitted gradient meets local differential privacy guarantees. However, since their method only shares the gradients for items that have been rated by the user, it indirectly gives away information about the user's interaction data. Shin~\etal~\cite{Shin2018} also take a distributed approach to recommendation and are the first to adopt \textit{user-level} local differential privacy as privacy requirement. They, however, have to deal with a large perturbation error caused by the stricter privacy requirements. They reduce the perturbation error by adopting dimensionality reduction through random projection, and sparse recovery, offering the first evidence of a recommender system working under local differential privacy guarantees. However, none of the above methods focus on recommendation for implicit feedback, and are only suitable for \textit{explicit} (i.e. ratings) data.
Recently, Duriakova~\etal\cite{Duriakova2019} proposed a decentralised matrix factorization protocol that enhances privacy by letting the user select the amount and type of information shared. However, the authors don't provide any formal differential privacy guarantees for their proposed method.

\subsection{Local Differential Privacy for Analytics and ML}
Because \textit{local} DP both eliminates the need for a central trusted data curator and defuses the risks of inference attacks, it has recently garnered more attention than its \textit{central} counterpart. Google's RAPPOR~\cite{Erlingsson2014} proposes a combination of randomized response and Bloom filters to make possible crowd-sourcing simple statistics, such as frequently visited websites or OS process names, under LDP guarantees. Similarly, Apple propose Private Count Mean Sketch~\cite{sketch_apple}, a data collection protocol to identify popular emojis and popular health data types from users, while guaranteeing local differential privacy. More generally, a series of LDP mechanisms have been proposed to privatize tasks like frequency-estimation and heavy-hitters detection~\cite{Zhu2020, Duchi, Wang2019, Nguyen2016}. 

Local differential privacy has also recently been applied to the task of federated learning~\cite{Liu2020, Sun2020, Truex2020} and, more broadly, to the task of distributed gradient descent~\cite{Nguyen2016, Shin2018, Wang2019}. To prevent private information leaking from raw gradient updates, each client feeds his local gradient updates to a privatization mechanism before shipping it to the server. This procedure, however, similarly to what we observed in our experiments, causes the estimation error of models learned on this privatized updates to increase with the dimensionality of the updates, raising utility issues even for very simple models. Several techniques have been proposed to increase the utility of these exchanges. Nguyên~\etal~\cite{Nguyen2016} propose sampling a \textit{single} dimension at random from the gradient vector, and submitting a perturbed version of the value to the server. The perturbed values are collected and the mean gradient of all users can be estimated on each gradient dimension. This technique succeeds in reducing the noise injected in the perturbed gradient update but, at the same time, induces extreme sparsity in communication between user and server. Shin~\etal~\cite{Shin2018} build on the previous method and adopt dimensionality reduction through random projection in order to increase the utility of the gradient data and to lower the estimation error. However, their method relies on sparse recovery to retrieve the full-dimensional gradient matrix from the projection, which is only suitable when the unperturbed gradient data is sparse, as in the case of recommendation based on explicit data.
With a similar objective, Liu~\etal~\cite{Liu2020} propose a top-k dimensions selection mechanism to send higher ``quality'' information to the server, while avoiding the reconstruction loss caused by random projection. Finally, Sun~\etal~\cite{Sun2020} propose splitting and shuffling gradient updates locally before sending them to the aggregator model to break linkability. However, the authors make the erroneous assumption that breaking linkability between the single gradient factors prevents the per-report privacy budgets from composing.

\section{Conclusions}
\label{sec:conclusions}

In this paper we presented what is to the best of our knowledge the first general privacy-preserving federated recommendations framework for implicit feedback under user-level local differential privacy guarantees. We empirically demonstrated the effectiveness of our federated recommendation framework across a variety of privacy and communication budgets, as well as item and user set cardinalities. On the MovieLens dataset, our system achieves up to HR@10=0.68 on 50,000 users with 5,000 items, and even on the full data set we show that our system can achieve good utility, with HR@10>0.5, without compromising user privacy. While likely not being an optimal solution for utility-first applications, where privacy is more of a concession than a requirement, we have shown how it is possible with a privacy-first approach to achieve reasonable performance on recommendation with implicit feedback.
\bibliographystyle{ACM-Reference-Format}
\bibliography{acmart.bib}

\end{document}